\documentclass{ecai} 

\usepackage{latexsym}
\usepackage{amssymb}
\usepackage{amsmath}
\usepackage{amsthm}
\usepackage{booktabs}
\usepackage{enumitem}
\usepackage{graphicx}
\usepackage{color}
\usepackage[frozencache,cachedir=minted-output]{minted}

\newtheorem{theorem}{Theorem}

\DeclareMathOperator*{\argmax}{arg\,max}

\def\hquad{\hskip.5em\relax}

\newcommand{\BibTeX}{B\kern-.05em{\sc i\kern-.025em b}\kern-.08em\TeX}

\begin{document}


\begin{frontmatter}




\title{Explaining an Agent's Future Beliefs through Temporally Decomposing Future Reward Estimators}


\author[A]{\fnms{Mark}~\snm{Towers}\orcid{0000-0002-2609-2041}\thanks{Corresponding Author. Email: mt5g17@soton.ac.uk}}
\author[B]{\fnms{Yali}~\snm{Du}\orcid{0000-0001-5683-2621}}
\author[A]{\fnms{Christopher}~\snm{Freeman}\orcid{0000-0003-0305-9246}} 
\author[A]{\fnms{Timothy}~J.~\snm{Norman}\orcid{0000-0002-6387-4034}}

\address[A]{School of Electronics and Computer Science, University of Southampton, UK}
\address[B]{Department of Informatics, Kings College London, UK}


\begin{abstract}
Future reward estimation is a core component of reinforcement learning agents; i.e., Q-value and state-value functions, predicting an agent's sum of future rewards. Their scalar output, however, obfuscates when or what individual future rewards an agent may expect to receive. We address this by modifying an agent's future reward estimator to predict their next $N$ expected rewards, referred to as Temporal Reward Decomposition (TRD). This unlocks novel explanations of agent behaviour. Through TRD we can: estimate when an agent may expect to receive a reward, the value of the reward and the agent's confidence in receiving it; measure an input feature's temporal importance to the agent's action decisions; and predict the influence of different actions on future rewards. Furthermore, we show that DQN agents trained on Atari environments can be efficiently retrained to incorporate TRD with minimal impact on performance.
\end{abstract}

\end{frontmatter}


\section{Introduction}
\label{sec:introduction}
With reinforcement learning (RL) agents exceeding human performance in complex and challenging game environments (e.g., Atari 2600 \cite{badia2020agent57}, DotA 2 \cite{berner2019dota}, and Go \cite{silver2017mastering}), there is significant interest in applying these methods to address practical problems, often in support of human judgement. There are several barriers to realising this vision, however, with the need for agents to be able to explain their decisions one of the most important \cite{qing2022survey}; agents need to be able to work with people \cite{dulac2021challenges}, and so we need effective Explainable Reinforcement Learning (XRL) mechanisms.

Central to RL agents is a future reward estimator (Q-value or state-value function) predicting the sum of future rewards for a given state. These functions are used either explicitly in the policy itself (e.g., DQN \cite{mnih2015human}) or for learning with a critic (e.g., PPO \cite{schulman2017proximal} and TD3 \cite{fujimoto2018addressing}). However, few XRL algorithms have devised methods to explain these functions directly. One problem is that their scalar outputs provide no information on its composition (i.e., when and what future rewards the agent believes it will receive), just its expected cumulative sum. 

\begin{figure}[h]
    \centering
    \includegraphics[width=\linewidth]{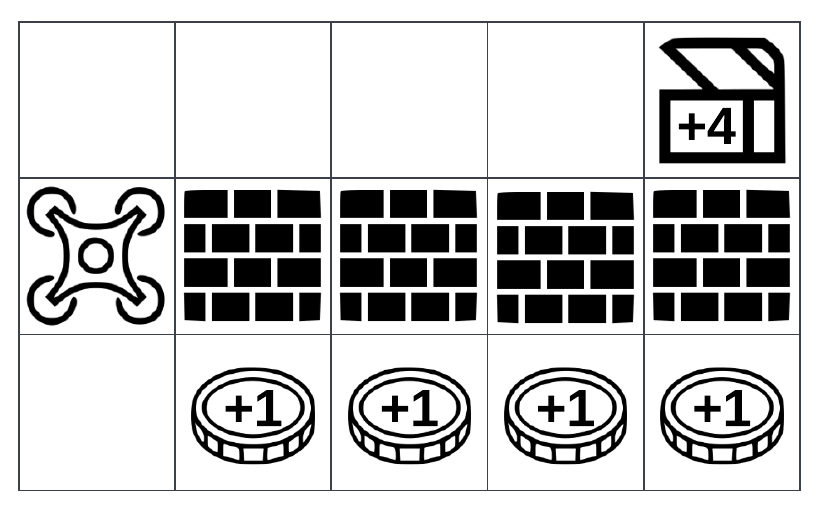}
    \caption{Example Gridworld with an agent and two paths (up and down) that contain different rewards.}
    \label{fig:gridworld}
    \vspace{1cm}
\end{figure}

An example of this problem is illustrated in Figure \ref{fig:gridworld} where a drone has two paths: up or down. Depending on the path taken, the drone can receive coins for 1 point each or the treasure chest for 4 points. Using a discount factor of $0.95$, the drone's Q-value for moving up is $3.26$ while moving down is $3.52$. Despite this small difference in Q-values, the quantity and temporal immediacy of their expected rewards are radically different; moving up, the drone receives a single large reward, while moving down receives many smaller rewards. A fact unknown from observing the Q-values alone although critical to agent behaviour in selecting whether to move up or down.

We propose a novel future reward estimator that predicts the agent's next $N$ expected rewards for a given state, referred to as Temporal Reward Decomposition (TRD) (Section \ref{sec:temporally-decomposed-q-value}). We prove that TRD is equivalent to the Q-value and state-value functions. In this way, TRD can report the temporal immediacy and quantity of future rewards for different action choices, enabling decisions to be explained and contrasted. For example, using Figure \ref{fig:gridworld}, the agent's TRD Q-value for moving down is $[0, 0.95, 0.90, 0.86, 0.81]$ and moving up $[0, 0, 0, 0, 3.26]$, enabling us to produce explanations such as ``while the sum of actual rewards is equal, taking the route down has more immediate rewards, which are preferred by the drone due to its discount factor''. 

Implementing TRD requires only two changes to a deep RL agent's future reward estimator: increase the network output by $N{+}1$ for predicting the future rewards and a novel element-wise loss function of future rewards (Section \ref{sec:temporally-decomposed-q-value}). Importantly, TRD can achieve similar performance as DQN \cite{mnih2015human} agents for Atari environments \cite{bellemare2013arcade} across a wide range of $N$ (Section \ref{sec:training-performance}).

Building on this direct access to an agent's predictions for individual future rewards, we explore three novel applications for understanding an agent's decision-making (Section \ref{sec:explainability}). The first is to generate explanations for an action choice based on when and what rewards it will receive, and, for particular environments, the agent's confidence in collecting a reward (Section \ref{subsec:extracting-reward-prob}). The second is to explain how the importance of an observation feature changes depending on how far into the future a reward is expected; e.g., identifying features that are more activate for earlier rewards (Section \ref{subsec:vis-temporal-decision-making}). Thirdly, we produce contrastive explanations using the difference in future expected rewards for two actions, which can reveal changes in expected rewards the agent will receive between them (Section \ref{subsec:contrast-exp}). Together, these three explanation mechanisms demonstrate the value of Temporal Reward Decomposition for XRL. 

Prior to presenting TRD, proving key properties and detailing how this can be used to generate explanations of an agent's action choices with respect to future rewards, we briefly review relevant prior research and present some preliminary formalisation that we build upon. 

\section{Related Work}
\label{sec:related-work}
In this brief review we focus on existing reward-based explanation mechanisms and algorithms for understanding an agent's decision-making in similar domains (see Qing \emph{et al.} \cite{qing2022survey} for a survey). 

Prior work in XRL has explored decomposing the Q-value into reward components and by future states. Juozapitis \emph{et al.}  \cite{juozapaitis2019explainable} proposed decomposing the future rewards into components. In Figure \ref{fig:gridworld}, for example, we have two components (or reward sources), the treasure chest and the coins. An explanation would then contrast the coin(s) versus treasure chest(s) along different trajectories, but not when any rewards are expected. Further work has incorporated policy summary \cite{septon2023integrating} and abstract action spaces for robotics \cite{lu2023closer} into the explanations. Alternatively, Tsuchiya \emph{et al.} \cite{tsuchiya2023explainable} proposed decomposing the Q-value into the importance of future states and Yau \emph{et al.} \cite{yau2020did} into the probabilities of state transitions. However, neither state decomposition proposal has been shown to scale to complex environments where explanations are most important for understanding agents. Importantly, all these decomposition approaches differ from our work as they require decomposing the reward estimator into components or states rather than over time, although future work could explore combining these approaches.

An alternative approach to understanding an agent's rewards is to modify the environment's reward function. Mirachandani and Karamcheti \cite{mirchandani2021ella} proposed a reward-shaping approach using natural language to convert long-horizon task descriptions to lower-level dense rewards. This allows the higher-level descriptions to be used to explain the agent policy however this relies on correctly interpreting these complex descriptors while our work requires no modification to the environment setup.  

For approaches that contribute similar applications to TRD, Madumal \emph{et al.} \cite{madumal2020explainable} proposed a text-based explanation using a hand-crafted causal model of a Starcraft 2 agent to explain why (or why not) to take an action with respect to environmental state variables from the causal model. Our explanation similarly illustrates the future reasoning of an agent, but does so in terms of rewards not changes to the environment's state. Most importantly, our approach requires neither a hand-crafted causal model nor explicitly identified environment features. 
To explain the agent's focus within an observation, Greydanus \emph{et al.} \cite{greydanus2018visualizing} proposed a perturbation-based saliency map that uses the changes in the policy output for noise applied to an area of the observation to understand a region's importance to the agent. This is limited to only visualising a region's importance for all future rewards, whereas combining TRD with saliency map algorithms can explain a particular future reward's regions of importance in decision-making.


Outside XRL, researchers have explored non-scalar variants of the Q-value, primarily for improving performance. Bellemare \emph{et al.} \cite{bellemare2017distributional} proposed C51, a training algorithm that learns the distribution of cumulative rewards rather than just the expectation, achieving state-of-the-art performance in Atari. Our work differs as we propose decompose the Q-value into the expected reward for future timesteps rather than the probability distribution over all future rewards. Furthermore, we propose new explanatory applications that are facilitated by TRD. 

\section{Preliminaries}
\label{sec:preliminary}
Before we present TRD, we provide sufficient technical detail on methods we build upon: Markov Decision Processes to mathematically describe TRD; Deep Q-learning for learning Q-value functions in complex environments; QDagger for learning with pretrained agents; and GradCAM for creating saliency map explanations.

To model a reinforcement learning environment, we use a Markov Decision Process \cite{puterman2014markov} described by the tuple $\langle S, A, R, P, T\rangle$. These variables denote the set of possible states and actions ($S$ and $A$ respectively), the reward function ($R(s, a)$) given a state action $s, a$ that is bounded to finite values, the transition probability ($P(s^\prime | s, a)$) of the next state ($s^\prime$) given the prior state-action ($s, a$) and the termination condition ($T(s)$) that returns True if the state ($s$) is a goal state. For simplicity, following Sutton and Barto \cite{sutton2018reinforcement}, we denote $S_i$, $A_i$ and $R_i$ as the state, action and reward received for timestep $i$.

Given an environment, we wish to learn a policy $\pi$ that maximises its cumulative rewards over an episode. Furthermore, to incentivise the agent to collect rewards sooner, we apply an exponential discount factor ($\gamma \in [0, 1)$). For a policy, $\pi$, we may define the expected sum of future rewards in terms of the Q-value, $q_\pi(s, a)$, or the state-value, $v_\pi(s)$, functions, Eqs.\ \eqref{eq:q-value} and \eqref{eq:state-value} respectively.

\begin{align}
    q_\pi(s, a) &= \mathbb{E}_{\pi} \Big[\sum^{\infty}_{n=0} \gamma^{n} R_{t+n} | S_t = s, A_t = a \Big] \label{eq:q-value}\\
    v_\pi(s) &= \mathbb{E}_{\pi} \Big[\sum^{\infty}_{n=0} \gamma^{n} R_{t+n} | S_t = s\Big] \label{eq:state-value}
\end{align}

To learn an optimal policy, agents can select actions that maximise the Q-value for a given state. Using this, Watkins and Dayan \cite{watkins1992q} proposed iteratively minimising the error between the predicted Q-value for a state-action and a bootstrapped target Q-value using the state-action's resultant reward plus the maximum Q-value in the next timestep (Eq.\ \eqref{eq:dqn-loss}), referred to as Q-learning. Importantly, given initial conditions and an infinite number of iterations, \cite{watkins1992q} proved Q-learning would converge to the optimal policy. 

\begin{align}
    L_{\text{Q}}(D) &= \mathbb{E}_{(s, a, R, s^\prime)\sim D} (q_{\pi}(s, a) - y_{\text{target}})^2 \label{eq:dqn-loss} \\ 
    y_{\text{target}} &= R + \gamma \max_{a^\prime \in A} \hat{q}_\pi (s^\prime, a^\prime) \label{eq:dqn-target}
\end{align}

This was extended by Mnih \emph{et al.} \cite{mnih2015human} to use neural networks, referred to as Deep Q-learning (DQN) for a general RL algorithm that achieved state-of-the-art performance in image-based environments. They combined several extensions to Q-learning including an experience replay buffer to store training examples, a target network for stability and a convolutional neural network to learn the Q-values. 

To help minimise the training time of TRD agents, we utilise QDagger \cite{agarwal2022reincarnating}, a workflow for learnt policies to reuse or transfer their knowledge to new agents. In particular, QDagger proposes two changes to an agent's training scheme: an offline training stage using a teacher's (pretrained agent) replay buffer and adding a distillation loss to the agent's (student) policy loss that minimises the KL divergence between the student's $\pi$ and teacher's $\pi_T$ policy  (Eq.\ \eqref{eq:qdagger}). The weighting of this loss term is controlled by $\lambda_T$ as the ratio of teacher to student average reward. Agarwal \emph{et al.} \cite{agarwal2022reincarnating} showed QDagger allows student agents to match the teacher's performance for Atari environments with 20x fewer observations than normal. 

\begin{equation}
    L_{\text{QDagger}}(D) = L_{Q}(D) + \lambda_T \mathbb{E}_{s \sim D} \left[ \sum_a \pi_T (a | s) \log \pi(a | s) \right]
    \label{eq:qdagger}
\end{equation}

We utilise GradCAM \cite{selvaraju2017grad}, a popular saliency map algorithm highlighting the input features that have the greatest influence on a neural network's decision-making. For a given convolutional layer, GradCAM computes the gradients from the layer's features to one of the network's outputs such that the gradient is proportional to the feature's importance in that network's decision-making for the output.

\section{Temporal Reward Decomposition}
\label{sec:temporally-decomposed-q-value}
As described in Section \ref{sec:introduction}, due to the scalar output of future reward estimators (i.e., Q-value and state-value functions), their reward composition cannot be known, preventing understanding when and what future rewards the agent expects to receive. We, therefore, propose a novel future reward estimator (Eqs.\ \eqref{eq:trd-array}), referred to as Temporal Reward Decomposition (TRD) that predicts an agent's next $N$ expected rewards. Furthermore, we prove its equivalence to scalar future reward estimators and provide a bootstrap-based loss function to learn the estimator (Eq.\ \eqref{eq:trd-loss}). For consistency, all equations in this section are for the Q-value with state-value equations in Appendix A. 


Before defining our TRD-based future reward estimators, to prove their equivalence to scalar future reward estimators (Eq.\ \eqref{eq:trd-equiv}), we first prove that the expected sum of future rewards is equivalent to the sum of expected future rewards enabling the decomposition of rewards in Eq.\ \eqref{eq:trd-array}: Theorem \ref{theorem:sum-expected-reward}.\footnote{Linearity of Expectation (LoE) is a property that any expectation can be split into its linear components, even for \emph{dependent} random variables \cite[Page 166]{stirzaker2003elementary}.}

\begin{equation}
    q^{\text{TRD}}_{\pi}(s, a) = \begin{pmatrix}
        \mathbb{E}_{\pi} [R_t | S_t = s, A_t = a] \\[3pt]
        \mathbb{E}_{\pi} [\gamma R_{t+1} | S_t = s, A_t = a] \\ 
        \vdots  \\ 
        \mathbb{E}_{\pi} [\gamma^{N-1} R_{t+N-1} | S_t = s, A_t = a] \\[3pt]
        \mathbb{E}_{\pi} \left[\sum_{i=N}^{\infty} \gamma^i R_{t+i} | S_t = s, A_t = a \right]
    \end{pmatrix}
    \label{eq:trd-array}
\end{equation}

\begin{equation}
    \sum q^{\text{TRD}}_{\pi}(s, a) \equiv q_{\pi}(s, a) \hspace{25pt} \forall s \in S, \forall a \in A
    \label{eq:trd-equiv}
\end{equation}

Using the notation in Section \ref{sec:preliminary}, we propose Eq.\ \eqref{eq:trd-array} that outputs a vector of the next $N$ expected rewards with the last element being equal to the cumulative sum of expected rewards from $N$ to $\infty$. Each element $i$ refers to the expected reward in $t+i$ timesteps with the final element being the sum of rewards beyond $N$ timesteps. Using Theorem \ref{theorem:sum-expected-reward}, Eq.\ \eqref{eq:trd-array} is provably equivalent to the scalar Q-value by summing over the array elements (Eq.\ \eqref{eq:trd-equiv}) through expanding Eq.\ \eqref{eq:expanded-sum-expectation} with $N{+}1$ expectations. Critically, this equivalence is not reversible such that given a scalar Q-value, Eq.\ \eqref{eq:trd-array} cannot be known.

\begin{theorem}
    Given a state $s$ and action $a$, the expected sum of rewards is equal to the sum of expected rewards, more precisely $\mathbb{E}_{\pi} \left[ \sum_{i=0}^{\infty} \gamma^i R_{t+i} | S_t = s, A_t = a \right] \equiv \sum_{i=0}^{\infty} \mathbb{E}_{\pi} [ \gamma^i R_{t+i} | S_t = s, A_t = a]$ for all $s \in S$ and $a \in A$.
    \label{theorem:sum-expected-reward}
\end{theorem}

\begin{proof}
    \begin{align}
        &\mathbb{E}_{\pi} \left[ \sum_{i=0}^{\infty} \gamma^i R_{t+i} \Big| S_t = s, A_t = a \right] \\
        =& \mathbb{E}_{\pi} \left[ R_{t} + \sum_{i=1}^{\infty} \gamma^i R_{t+i} \Big| S_t = s, A_t = a \right] \\
        =& \mathbb{E}_{\pi} [R_t | S_t = s, A_t = a] \nonumber \\ 
        &\hquad + \mathbb{E}_{\pi} \left[ \sum_{i=1}^{\infty} \gamma^i R_{t+i} \Big| S_t = s, A_t = a \right ] \hspace{5pt}\text{(given LoE\footnotemark[1])} \\
        =& \mathbb{E}_{\pi} [R_t | S_t = s, A_t = a] + \mathbb{E}_{\pi} [\gamma R_{t+1} | S_t = s, A_t = a] \nonumber \\ 
            &\qquad\qquad+ \mathbb{E}_{\pi} \left[ \sum_{i=2}^{\infty} \gamma^i R_{t+i} \Big| S_t = s, A_t = a \right] \label{eq:expanded-sum-expectation} \\ 
        =& \sum_{i=0}^{\infty} \mathbb{E}_{\pi} [\gamma^i R_{t+i} | S_t = s, A_t = a] 
    \end{align}
    \label{proof:sum-expected-reward}
\end{proof}

Implementing TRD within a deep RL agent's future reward estimator requires two primary changes. The first is increasing the neural network output by $N{+}1$; i.e., the size of Eq.\ \eqref{eq:trd-array} for predicting the next $N$ future rewards. The second is the loss function (Eq.\ \eqref{eq:trd-loss}) for the network to learn Eq.\ \eqref{eq:trd-array}. Additionally, as the network now outputs a vector of future rewards rather than a scalar, for action selection and other applications, $q_\pi$ can be recovered by summing across vector elements before being applied as normal (Eq.\ \eqref{eq:trd-equiv}). Appendix B includes pseudocode for implementing the loss function, and the associated GitHub repository\footnote{\url{https://github.com/pseudo-rnd-thoughts/temporal-reward-decomposition}} contains the implementation of a TRD-modified DQN training algorithm using Gymnasium \cite{towers2024gymnasium}. 

\begin{figure*}[ht]
    \begin{equation}
        q^{\text{TRD}}_{\pi}(s, a) = \begin{pmatrix}
            \mathbb{E}_{\pi} [R_t | S_t = s, A_t = a] + 
            \dots + \mathbb{E}_{\pi} [\gamma^{w-1} R_{t+w-1} | S_t = s, A_t = a] \\[4pt]
            \mathbb{E}_{\pi} [\gamma^{w} R_{t+w} | S_t = s, A_t = a] + 
            \dots + \mathbb{E}_{\pi} [\gamma^{2w-1} R_{t+2w-1} | S_t = s, A_t = a] \\
            \vdots \\
            \sum_{i=(N-1)w}^{Nw} \mathbb{E}_{\pi} [\gamma^i R_{t+i} | S_t = s, A_t = a] \\[4pt]
            \mathbb{E}_{\pi} [\sum_{i=Nw}^{\infty} \gamma^i R_{t+i} | S_t = s, A_t = a] 
        \end{pmatrix}
        \label{eq:trd-binned-array}
    \end{equation}

    \begin{equation}
        \sum q^{\text{TRD}}_{\pi}(s, a) \equiv q_{\pi}(s, a)
        \hspace{30pt} \forall s \in S, \forall a \in A
        \label{eq:grouped-trd-equiv}
    \end{equation}
\end{figure*}

\begin{figure*}
    \begin{equation}
        L_{\text{TRD}} = \mathbb{E}_{(s_t, a_t, R_{t+i}, s_{t+w})\sim D} \begin{bmatrix}
            \left( q^{\text{TRD}_0}_\pi (s_t, a_t) - \sum_{i=0}^w R_{t+i} \right)^2 \\[4pt] 
            \left( q^{\text{TRD}_1}_\pi (s_t, a_t) - \gamma^w q^{\text{TRD}_0}_\pi (s_{t+w}, a^\prime) \right)^2 \\[4pt]
            \left( q^{\text{TRD}_2}_\pi (s_t, a_t) - \gamma^w q^{\text{TRD}_1}_\pi (s_{t+w}, a^\prime) \right)^2 \\
            \vdots \\
            \left( q^{\text{TRD}_{N}}_\pi (s_t, a_t) - \gamma^w q^{\text{TRD}_{N-1}}_\pi (s_{t+w}, a^\prime) \right)^2 \\[4pt]
            \left( q^{\text{TRD}_{N+1}}_\pi (s_t, a_t) - \gamma^w (q^{\text{TRD}_{N}}_\pi (s_{t+w}, a^\prime) + q^{\text{TRD}_{N+1}}_\pi (s_{t+w}, a^\prime)) \right)^2
        \end{bmatrix}
        \label{eq:trd-loss}
    \end{equation}
\end{figure*}

For long-horizon environments where an agent may take hundreds or thousands of actions, TRD is limited in scale as the number of predicted rewards scales linearly with the number of output neurons. We therefore propose an alternative approach to preserve the temporal distance that can be explained using a fixed number of output neurons. Rather than each vector element predicting an individual reward, Eq.\ \eqref{eq:trd-binned-array} groups rewards in each vector element; e.g., for pair grouping $[R_{t} + R_{t+1}, R_{t+2} + R_{t+3}, \dots]$. This approach, denoted $w$ for the reward grouping size, scales linearly with the number of future rewards by $w$ for a fixed $N$ such that the total number of predicted rewards is $N \cdot w$. Importantly, like Eq.\ \eqref{eq:trd-array}, Eq.\ \eqref{eq:trd-binned-array} is equivalent to the Q-value by summing across elements (Eq.\ \eqref{eq:grouped-trd-equiv}) using $N \cdot w + 1$ expansions of Eq.\ \eqref{eq:expanded-sum-expectation}. Additionally, for $w=1$, Eq.\ \eqref{eq:trd-binned-array} is equivalent to Eq.\ \eqref{eq:trd-array} and implementation only requires utilising an N-step \cite{sutton2018reinforcement} experience replay buffer to compute the sum of the first $w$ rewards and the next observation in $w$ timesteps. 

As a result, $N$ and $w$ present a trade-off between the reward vector size ($N$) and precise knowledge of each timestep's expected reward ($w$). For example using Figure \ref{fig:gridworld}, if $w=2$ and $N=2$ then the $q^{\text{TRD}}_\pi$ for moving up is $[0, 0, 3.26]$ as $[0+0, 0+0, 3.26]$ and moving down is $[0.95, 1.76, 0.81]$ as $[0+0.95, 0.90 + 0.86, 0.81]$. Furthermore, to predict, for example, the next 30 rewards, $N=30, w=1$ and $N=6, w=5$ are both valid parameters. We explore the impact of these parameters on training in Section \ref{sec:training-performance}.

\begin{figure*}[t]
    \centering
    \includegraphics[width=\linewidth]{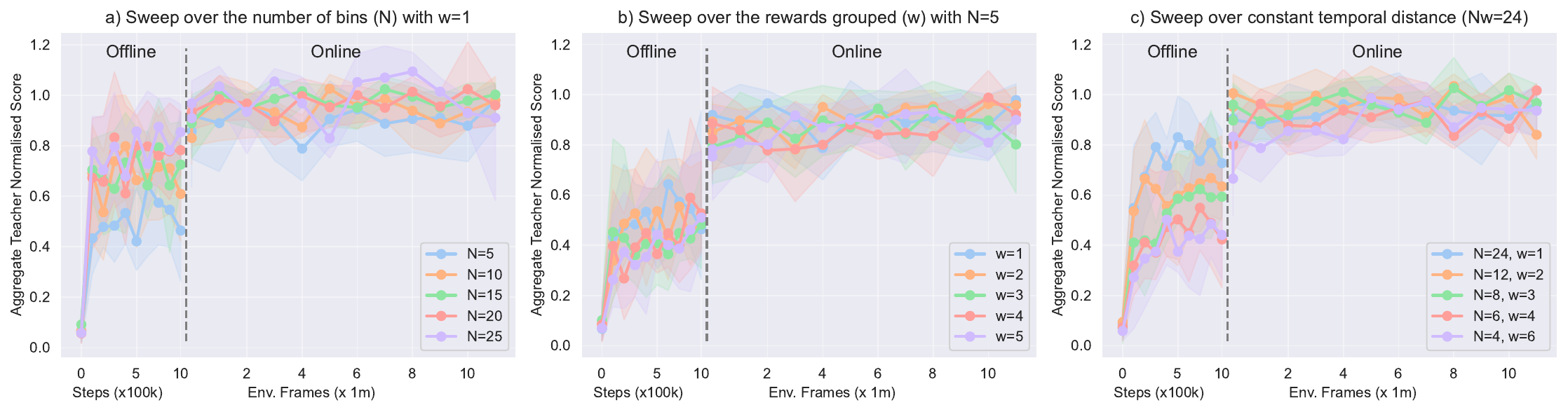}
    \caption{Interquantile mean training curves for Atari TRD-DQN agents for three environments (Breakout, Space Invaders and Ms Pacman) with three repeats, normalised by the teacher's score. Offline and Online indicate where training used the offline replay buffer and the online environment steps.}
    \label{fig:dqn-sweep-episodic-return}
    \vspace{1cm}
\end{figure*}

\begin{figure*}[t]
    \centering
    \includegraphics[width=\linewidth]{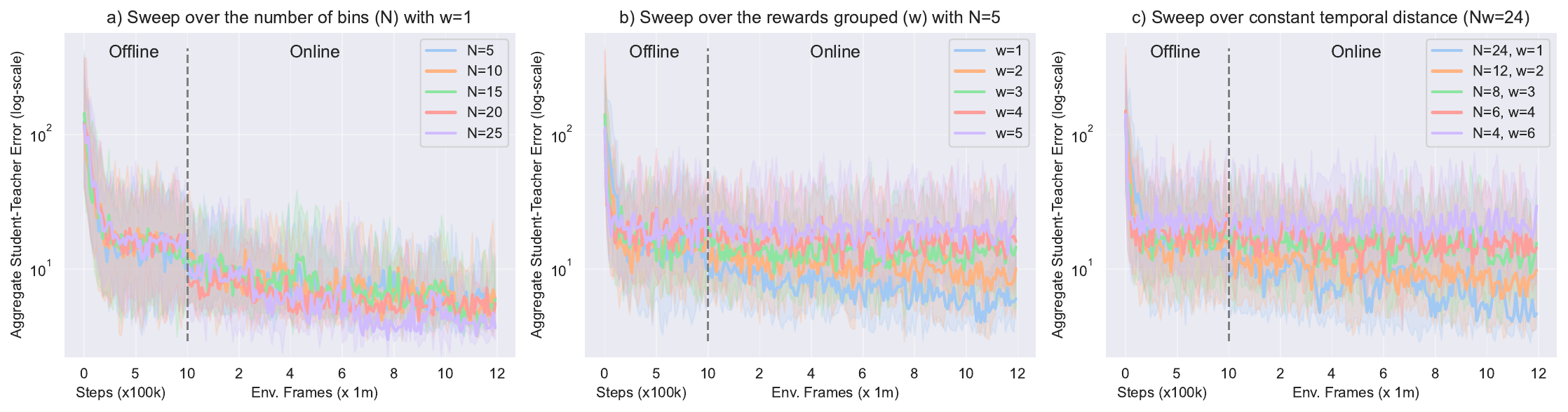}
    \caption{The Mean Squared Error between the student (TRD agent) and pretrained teacher agent averaged over three Atari environments with three repeats. Offline and Online indicate where training used the offline replay buffer and the online environment steps.}
    \label{fig:dqn-sweep-error}
    \vspace{1cm}
\end{figure*}

Through experiment, we found that converting $q^{\text{TRD}}_\pi$ to $q_\pi$ by summing over elements (Eq.\ \eqref{eq:trd-equiv}), then using the scalar loss function (Eq.\ \eqref{eq:dqn-loss}) does not converge to $q^{\text{TRD}}_\pi$. Therefore, based on the Q-learning loss function (Eq.\ \eqref{eq:dqn-loss}), we define a novel element-wise mean squared error of reward vectors (Eq.\ \eqref{eq:trd-loss}) where $a^\prime$ denotes the optimal next action ($\argmax_{a \in A} \sum q^{\text{TRD}}_{\pi}(s_{t+w}, a)$) and we use the following notation to index an element of the reward vector:
\begin{align}
    q^{\text{TRD}_0}_\pi(s, a) &= \mathbb{E}_{\pi} [R_t | S_t = s, A_t = a] \\
    q^{\text{TRD}_1}_\pi(s, a) &= \mathbb{E}_{\pi} [\gamma R_{t+1} | S_t = s, A_t = a] \\
    q^{\text{TRD}_N}_\pi (s, a) &= \mathbb{E}_{\pi} [\gamma^{N-1} R_{t+N-1} | S_t = s, A_t = a] \\
    q^{\text{TRD}_{N+1}}_\pi (s, a) &= \mathbb{E}_{\pi} \left[\sum_{i=N}^{\infty} \gamma^i R_{t+i} | S_t = s, A_t = a\right] 
\end{align}

For Eq. \eqref{eq:trd-loss}, we construct a predicted and bootstrap-based target value (cf.\ Q-learning), computing the element-wise mean squared error of the predicted and target reward vectors. The prediction is the reward vector for the action taken in state $t$, $q_\pi^{\text{TRD}}(s_t, a_t)$. For the target, the first element is the actual reward collected ($R_t$ to $R_{t+w}$) followed by the reward vector for the optimal action in $s_{t+w}$, $q_\pi^{\text{TRD}}(s_{t+w}, a^\prime)$, shifted along one position with the last two elements combined. We do this because element $i$ of the reward vector, $q^{\text{TRD}_i}_\pi (s_t, a_t)$, refers to the predicted reward in $t{+}i$ timesteps, for the next observation, $t{+}w$, the equivalent reward vector element is $i{-}1$ in the target vector, $\gamma q^{\text{TRD}_{i-1}}_\pi (s_{t+w}, a^\prime)$. 

\section{Retraining Pretrained Agents for TRD}
\label{sec:training-performance}
The goal of Temporal Reward Decomposition (TRD) is to provide information about an agent's expected future rewards over time so that we can use this information to better understand its behaviour. For this to be practically effective, TRD agents should be capable of achieving performance similar to their associated base RL agent. In this section, therefore, we evaluate the performance of DQN agents \cite{mnih2015human} that incorporate TRD for a range of Atari environments \cite{bellemare2013arcade} and assess the impact of TRD's two hyperparameters on training: reward vector size, $N$; and reward grouping, $w$.

We conduct hyperparameter sweeps across each independently, varying $N$, $w$, and $N \cdot w$, across three Atari environments (Breakout, Space Invaders and Ms. Pacman), each containing different reward functions. To account for variability in training, we repeat our hyperparameter sweeps three times. The training hyperparameters and hardware used in training, along with the agent's final scores, are presented in Appendix B. Training scripts and final neural network weights for all runs are provided in the associated GitHub repository.

Rather than training agents from scratch for these environments, we use open-sourced pretrained Atari agents \cite{Huang_Open_RL_Benchmark_2024} and the QDagger training workflow \cite{agarwal2022reincarnating}, described in Section \ref{sec:preliminary}. 

Using periodic evaluation on the same ten seeds, Figure \ref{fig:dqn-sweep-episodic-return} plots the teacher normalised interquartile mean \cite{agarwal2021deep} of the episodic reward. We find that all three hyperparameter sweeps enable the agent to approach the pretrained (teacher) agent's score with neither parameter having a significant detrimental impact. Only the offline training for a constant temporal distance ($N \cdot w = 24$) do the agents with smaller values of $w$ showcase greater initial performance, but this difference is resolved during the online training stage.  

Interestingly, for the sweep of $N$, we found no degradation in performance, which was unexpected as we believed that larger values of $N$ would require more training to reach the same performance. As a result, in Section \ref{sec:explainability}, we trained agents with $N{=}40, w{=}1$. Further work is required to understand if these performance curves hold for larger values of $N$ and for more complex environments or agents. 

To verify that our TRD loss function (Eq. \eqref{eq:trd-loss}) converges to a policy that is similar to the pretrained agent's scalar Q-value. Figure \ref{fig:dqn-sweep-error} plots the mean squared error of the Q-values for both pretrained DQN agents and TRD agents during training. We find all parameters get close to the pretrained agent's Q-value with $w{=}1$ being an important factor.

Regarding the computation impact of incorporating TRD, we found that our QDagger+TRD DQN agents took $\approx10\%$ fewer steps per second than our base DQN agents, $248$ to $274$ steps per second, respectively. This performance will be jointly caused by QDagger requiring an additional forward pass from the teacher agent and TRD using a larger network output and a more complex loss function. 

\section{Explaining an Agent's Future Beliefs and Decision-Making}
\label{sec:explainability}
We now present three novel explanation mechanisms using future expected rewards: understanding what rewards the agent expects to receive and when, and their confidence in this prediction; visualising an observation feature's importance for predicting rewards at near and far timesteps; and a contrastive explanation using the difference in future rewards to understand the impact of different actions choices (Sections \ref{subsec:extracting-reward-prob}, \ref{subsec:vis-temporal-decision-making}, and \ref{subsec:contrast-exp} respectively). We showcase these applications using three different Atari environments with more examples in Appendices C, D, and E. All agents were retrained DQN agents incorporating TRD using $N{=}40$ and $w{=}1$. 

\begin{figure*}[ht]
    \centering
    \includegraphics[width=0.85\linewidth]{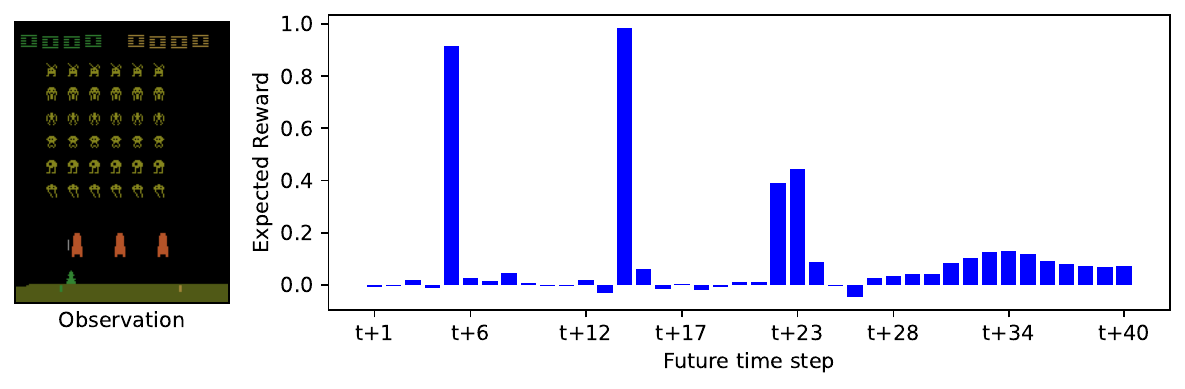}
    \caption{A Space Invaders observation (left) with the respective predicted next 40 future expected rewards (right).}
    \label{fig:spaceinvaders-expected-reward}
    \vspace{1cm}
\end{figure*}

\begin{figure*}[ht]
    \centering
    \includegraphics[width=0.85\linewidth]{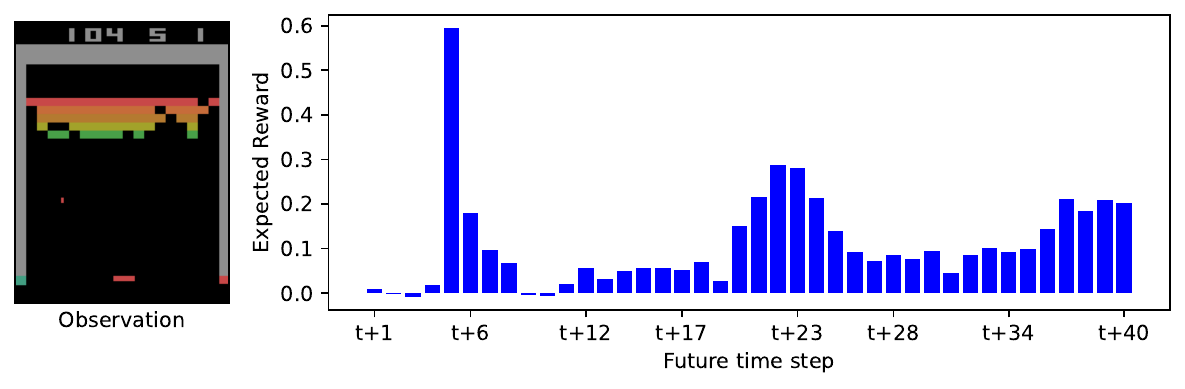}
    \caption{A Breakout observation (left) with the respective predicted next 40 future expected rewards (right).}
    \label{fig:breakout-expected-reward}
    \vspace{1cm}
\end{figure*}

\begin{figure*}[ht]
    \centering
    \includegraphics[width=.95\linewidth]{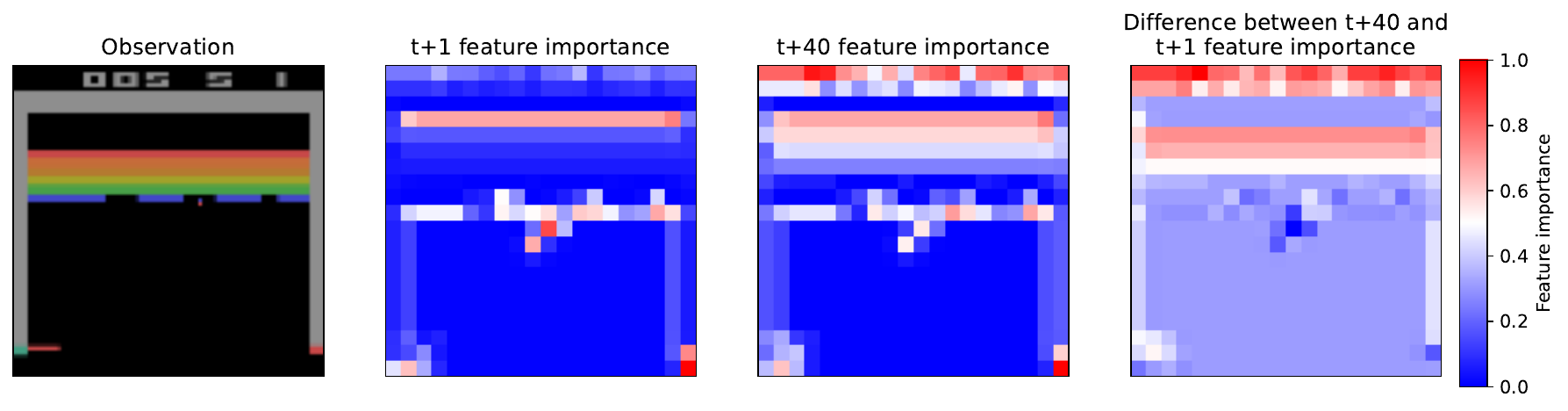}
    \caption{GradCAM saliency maps for the $t+1$ and $t+40$ expected reward along with their difference for a Breakout observation. GradCAM uses the first convolutional layer of the agent's neural network to differentiate.}
    \label{fig:breakout-feature-importance}
    \vspace{1cm}
\end{figure*}

\begin{figure*}
    \centering
    \includegraphics[width=.95\linewidth]{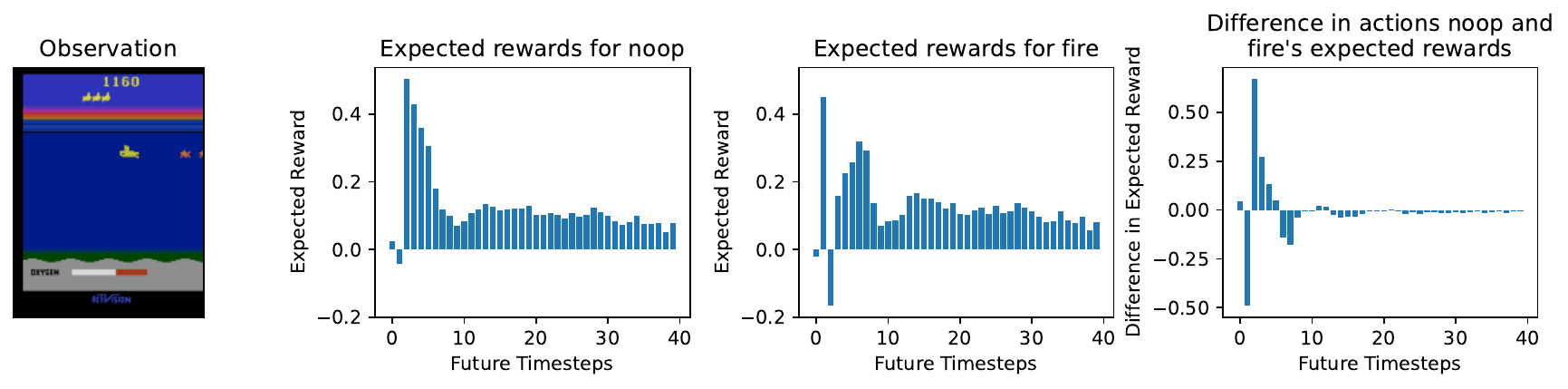}
    \caption{The difference of each future expected reward for taking Left and Right actions of the observation for the Atari Seaquest environment.}
    \label{fig:action-expected-reward-diff}
    \vspace{1cm}
\end{figure*}

\subsection{What Rewards to Expect and When?}
\label{subsec:extracting-reward-prob}
For environments with complex transition dynamics or reward functions such as Atari, understanding how an agent maximises its future rewards or predicting what rewards it will receive and when is not possible, unlike with the toy example illustrated in Figure \ref{fig:gridworld}. We show here how a TRD agent's predicted future rewards supply this information, presenting an important explanatory perspective for understanding agent decisions. 
Furthermore, for environments with binary reward functions (i.e., where the rewards are either zero or a constant value) the agent's expected reward can be further decomposed into the probability of the reward function components. Atari uses integer rewards and DQN agents clip rewards to -1 to 1, and so for these examples the agent's probability of collecting a reward is equivalent to the reward's expectation.

Figures \ref{fig:spaceinvaders-expected-reward} and \ref{fig:breakout-expected-reward} plot the agent's expected rewards over the next 40 timesteps for the observation on the left. As $w{=}1$, the discount factor is constant for each predicted timestep, and so we factor it out, leaving just the expected reward. Without domain knowledge of each environment and its reward functions, we can observe from the expected rewards plots that the agent expects periodic non-zero rewards every 8 to 9 timesteps for Space invaders and every 15 timesteps for Breakout. Additionally, considering that the expected rewards (for these environments) are equivalent to the agent's confidence (probability) in receiving a reward for a particular timestep, users can infer that the agent's confidence reduces over time for the specific timestep that the agent will receive a reward. As such, for space invaders, the agent has high confidence for the close timesteps ($t{+}6$ and $t{+}15$) with the expected rewards for the third and fourth rewards being distributed across several timesteps ($t{+}23$ to $t{+}24$ and $t{+}30$ to $t{+}40$).

Further, utilising domain knowledge of each environment, Figures \ref{fig:spaceinvaders-expected-reward} and \ref{fig:breakout-expected-reward} correlate with our understanding as agents can only shoot aliens or break bricks periodically. Additionally, as the policy is stochastic due to epsilon-greedy action selection and with randomness in the environment, the uncertainty of timesteps far in the future is unsurprising and matches with human expectations. 

Building on the two figures, we can generate videos of the agent's expected rewards across entire episodes plotting the expected reward for each observation. Example videos are provided in the associated GitHub repository and contain significant additional context for users to visualise how the agent's predicted future rewards change over time as the environment's state evolves. 

As a result, we anticipate that TRD has the potential to aid researchers and developers debug RL agents; Figure \ref{fig:spaceinvaders-expected-reward} and the related videos provide novel information about an agent's future beliefs and its understanding of an environment's reward function. 

\subsection{What Observation Features are Important?}
\label{subsec:vis-temporal-decision-making}
Understanding the areas of an input that have the greatest impact on a neural network is a popular technique for generating explanations, called saliency maps. These allow users to visualise what features of an observation most influence an agent's decision. With access to an agent's beliefs about its future expected rewards, TRD provides novel saliency map opportunities to understand how the agent's focus with respect to an observation varies. 

Utilising GradCAM \cite{selvaraju2017grad} (a popular saliency map algorithm described in Section \ref{sec:preliminary}), we can select individual expected rewards as the output to discover its feature importance. Figure \ref{fig:breakout-feature-importance} plots an Atari Breakout observation and the normalised feature importance for the expected reward of the next timestep ($t{+}1$) and the most distance expected reward ($t{+}40$) along with their normalised difference. The feature importance plots highlight areas of focus (red), influencing its decision and ignored areas (blue). We find that the agent's focus on the ball and bricks vary depending on how far in the future a reward is predicted. For the $t{+}1$ feature importance, the agent is highly focused (shown in red) on the ball in the centre. In comparison, for $t{+}40$, the agent has a greater focus on the bricks than the ball. Using domain knowledge of the environment validates human expectations as the number of bricks left and their position will have greater long-term importance to the agent than the ball. This difference is highlighted when subtracting the feature importance of $t{+}1$ from $t{+}40$ such that the ball's importance is significantly lower (shown in blue) and the bricks have relatively greater importance (shown in red). 

To help visualise this change in an observation feature's importance across each predicted future reward, we provide a video of Figure \ref{fig:breakout-feature-importance} within the associated GitHub repository. Additionally, we provide a video of an episode plotting the first and last predicted reward's feature importance for each timestep. Like visualising an agent's expected reward in Section \ref{subsec:extracting-reward-prob}, Figure \ref{fig:breakout-feature-importance} and videos can help researchers and developers understand in what context a feature has importance for an agent. Previously, it was only possible to understand a feature's importance to predict the agent's total reward, whereas TRD provides us with the ability to investigate the importance of features in a more granular way. 

\subsection{What is the Impact of an Action Choice?}
\label{subsec:contrast-exp}
Within an environment, there are often multiple (possibly similar) paths to complete a goal with humans interested in understanding the differences between them (e.g., Figure \ref{fig:gridworld}). Contrastive explanations are a popular approach to understanding the reasons for taking one decision over another. In our case, this is the choice between two alternative actions in some state \cite{miller2021contrastive}. With the future expected rewards, TRD provides additional information to compare and contrast states and actions using what rewards the agent expects to receive and when along different paths. In this section, we show how simple explanations only using the timestep-wise difference in expected rewards can help understand an action's impact on an agent's future rewards.

Figure \ref{fig:action-expected-reward-diff} shows the expected reward for taking no action (noop) and firing and the differences between the expected reward for noop and firing in the Atari Seaquest environment. The right-hand side figure shows that the difference in future rewards produces a positive and negative spike after which the expected rewards converge. We can infer from these spikes that if the agent fires rather than noop then there is a more immediate reward, whereas if the agent waits, taking no action, the reward is delayed resulting in a later spike. Crucially, this difference in reward outcomes is resolved afterwards causing no long-term difference in the agent's expected rewards. Using domain knowledge, we can assume that this means if the agent doesn't fire in this timestep, it will most likely fire in the following timestep or soon after, thus receiving a slightly delayed reward.

Collectively, with the explanations from Sections \ref{subsec:extracting-reward-prob} and \ref{subsec:vis-temporal-decision-making}, contrastive explanations highlight the consequences of different actions on an agent's future rewards.

\section{Conclusion}
\label{sec:conclusion}
Temporal Reward Decomposition (TRD) is a novel reward estimator that is equivalent to scalar future reward estimators that can uniquely reveal additional information about deep reinforcement learning agents. We have shown that pretrained Atari agents can be efficiently retrained to incorporate TRD with minimal impact on performance. Furthermore, we have showcased three novel explanatory mechanisms enabled by TRD, demonstrating how these can aid researchers and developers understanding agent behaviour in complex environments such as the three Atari environments considered here. We can ask ``What rewards to expect and when?'' by predicting rewards numerous timesteps into the future and the confidence the agent has in their prediction (Section \ref{subsec:extracting-reward-prob}). We can ask ``What observation features are important?'' by revealing how an agent's focus changes depending on the immediacy of the reward predicted (Section \ref{subsec:vis-temporal-decision-making}). Lastly, we can as ``What is the impact of an action choice?'' by revealing the difference in future expected rewards for two alternative actions (Section \ref{subsec:contrast-exp}). 

TRD can be extended in various ways to better explain an agent's future rewards. Incorporating prior decomposition approaches such as Juozapaitis \emph{et al.} \cite{juozapaitis2019explainable} to explain the future expected rewards of different reward components is a clear option for future research. Further, each reward could be modelled as a probability distribution, decomposing the expectation of a reward for a timestep \cite{bellemare2017distributional}. A further avenue for future research is that the linear relationship between future reward estimators and TRD (Eq.\ \eqref{eq:trd-equiv}) may be exploited for more efficient training.  




\begin{ack}
This work was supported by the UKRI Centre for Doctoral Training in Machine Intelligence for Nano-electronic Devices and Systems [EP/S024298/1] and RBC Wealth Management.

Thanks to John Birbeck for advice on the proof of Theorem \ref{theorem:sum-expected-reward}. 
\end{ack}



\bibliography{m2088}

\newpage
\appendix

\section{State-value based Temporal Reward Decomposition}
Section 4 outlines the Temporal Reward Decomposition for the Q-value only for consistency. Therefore, in this Appendix, we provide the equivalent Temporal Reward Decomposition for the state-value with theorem \ref{theorem:state-value-sum-expected-reward} providing the equivalent given only a state $s$, Eqs. \eqref{eq:state-value-trd-array} and \eqref{eq:state-value-trd-binned-array} for the state-value based reward vector and Eq. \eqref{eq:state-value-trd-loss} for the loss function. 

\begin{theorem}
    Given a state $s$, the expected sum of rewards is equal to the sum of expected rewards, more precisely $\mathbb{E}_{\pi} \left[ \sum_{i=0}^{\infty} \gamma^i R_{t+i} \right | S_t = s] \equiv \sum_{i=0}^{\infty} \mathbb{E}_{\pi} [ \gamma^i R_{t+i} | S_t = s]$.
    \label{theorem:state-value-sum-expected-reward}
\end{theorem}

\begin{proof}
    \begin{align}
        &\mathbb{E}_{\pi} \left[ \sum_{i=0}^{\infty} \gamma^i R_{t+i} \Big| S_t = s\right] \\
        =& \mathbb{E}_{\pi} \left[ R_{t} + \sum_{i=1}^{\infty} \gamma^i R_{t+i} \Big| S_t = s \right] \\
        =& \mathbb{E}_{\pi} [R_t | S_t = s] + \mathbb{E}_{\pi} \left[ \sum_{i=1}^{\infty} \gamma^i R_{t+i} \Big| S_t = s \right ] \\
        =& \mathbb{E}_{\pi} [R_t | S_t = s] + \mathbb{E}_{\pi} [\gamma R_{t+1} | S_t = s] + \nonumber \\ 
        & \qquad\qquad \mathbb{E}_{\pi} \left[ \sum_{i=2}^{\infty} \gamma^i R_{t+i} \Big| S_t = s \right] \\ 
         =& \sum_{i=0}^{\infty} \mathbb{E}_{\pi} [\gamma^i R_{t+i} | S_t = s] 
    \end{align}
\end{proof}

\begin{equation}
    v^{\text{TRD}}_{\pi}(s) = \begin{pmatrix}
        \mathbb{E}_{\pi} [R_t | S_t = s] \\[3pt]
        \mathbb{E}_{\pi} [\gamma R_{t+1} | S_t = s] \\ 
        \vdots  \\ 
        \mathbb{E}_{\pi} [\gamma^{N-1} R_{t+N-1} | S_t = s] \\[3pt]
        \mathbb{E}_{\pi} \left[\sum_{i=N}^{\infty} \gamma^i R_{t+i} | S_t = s \right]
    \end{pmatrix}
    \label{eq:state-value-trd-array}
\end{equation}

\section{Training}
For training, in this Appendix, we provide the training hardware used, agent performance, an example implementation of the TRD-based loss function (Listing \ref{listing:trd-training}), and a Table of training hyperparameters (Table \ref{tab:hyperparameters}).

For the hardware used for each training script, we used a single Nvidia V100 and Intel Xeon Gold 6138 with 20 cores, resulting in a wall time of around 7 hours. 

As we conduct a hyperparameter sweep, we report the average final evaluation results for the training environments using $N=25, w=1$ (and a seed of 0). The scores are $399.0, 1444.0$, and $2756.0$ for Breakout, Space Invaders, and Ms Pacman, respectively (the teacher-based DQN agents have scores of $386.9, 1324.0$, and $2433.0$) with clipped rewards.  

\begin{figure}[ht]
    \begin{minted}[breaklines]{python}
import numpy as np

def trd_loss(obs, actions, next_obs, 
             rewards, terminations):
    next_q_values = model(next_obs)
    next_actions = np.argmax(np.sum(next_q_values))
    next_q_values = next_q_values[next_actions]
    
    q_targets = (1 - terminations) * discount_factor * next_q_values
    q_targets = np.roll(q_targets, shift=1)
    q_targets[-1] += q_targets[0]
    q_targets[0] = rewards

    q_values = model(obs)
    q_actions = q_values[actions]

    loss = np.mean(np.square((q_actions - q_targets)))
    return loss
    \end{minted}
    \caption{Example implementation of the TRD Q-value loss function}
    \label{listing:trd-training}
    \vspace{1cm}
\end{figure}

\begin{figure*}[ht]
    \begin{equation}
        v^{\text{TRD}}_{\pi}(s) = \begin{pmatrix}
            \mathbb{E}_{\pi} [R_t | S_t = s] + 
            \dots + \mathbb{E}_{\pi} [\gamma^{w-1} R_{t+w-1} | S_t = s] \\[4pt]
            \mathbb{E}_{\pi} [\gamma^{w} R_{t+w} | S_t = s] + 
            \dots + \mathbb{E}_{\pi} [\gamma^{2w-1} R_{t+2w-1} | S_t = s] \\
            \vdots \\
            \sum_{i=(N-1)w}^{Nw} \mathbb{E}_{\pi} [\gamma^i R_{t+i} | S_t = s] \\[4pt]
            \mathbb{E}_{\pi} [\sum_{i=Nw}^{\infty} \gamma^i R_{t+i} | S_t = s] 
        \end{pmatrix}
        \label{eq:state-value-trd-binned-array}
    \end{equation}

    \begin{equation}
        \sum v^{\text{TRD}}_{\pi}(s) \equiv v_{\pi}(s)
        \hspace{30pt} \forall s \in S
    \end{equation}
\end{figure*}

\begin{figure*}
    \begin{equation}
        L_{\text{TRD}} = \mathbb{E}_{(s_t, R_{t+i}, s_{t+w})\sim D} \begin{bmatrix}
            \left( v^{\text{TRD}_0}_\pi (s_t) - \sum_{i=0}^w R_{t+i} \right)^2 \\[4pt] 
            \left( v^{\text{TRD}_1}_\pi (s_t) - \gamma^w v^{\text{TRD}_0}_\pi (s_{t+w}) \right)^2 \\[4pt]
            \left( v^{\text{TRD}_2}_\pi (s_t) - \gamma^w v^{\text{TRD}_1}_\pi (s_{t+w}) \right)^2 \\
            \vdots \\
            \left( v^{\text{TRD}_{N}}_\pi (s_t) - \gamma^w v^{\text{TRD}_{N-1}}_\pi (s_{t+w}) \right)^2 \\[4pt]
            \left( v^{\text{TRD}_{N+1}}_\pi (s_t) - \gamma^w (v^{\text{TRD}_{N}}_\pi (s_{t+w}) + v^{\text{TRD}_{N+1}}_\pi (s_{t+w})) \right)^2
        \end{bmatrix}
        \label{eq:state-value-trd-loss}
    \end{equation}
\end{figure*}

For Listing \ref{listing:trd-training}, $N$ and $w$ are called \mintinline{python}{num_bins} and \mintinline{python}{reward_width} respectively. To implement the loss function, there are two core changes to a standard DQN-based loss function: selecting the \mintinline{python}{next_q_values} and determining the \mintinline{python}{q_targets}. For computing the \mintinline{python}{next_q_values} is \mintinline{python}{np.max(next_q_values, axis=-1)} however, due to the output including the future expected rewards, we first need to roll up the rewards to compute the scalar Q-value to know the optimal next actions, which can be used to collect the optimal future expected rewards for each observation. The second change is to compute the \mintinline{python}{q_targets} which is normally \mintinline[breaklines]{python}{(1 - terminations) * discount_factor * next_q_values}. As the number of future rewards is variable, we modify the multiple the discount factor by the reward width. Next, we \mintinline{python}{np.roll} the targets by $1$; this is equivalent to moving the first index to the second index, second to third, etc, and the last index to the first. Using the shifted predicted future rewards; we update the last element to include the new first element (previously the last element of the predicted rewards). Finally, we set the first element as the actual rewards collected by the agent. 

\begin{table}
    \centering
    \caption{Training hyperparameters for DQN Atari with TRD and QDagger heuristics}
    \label{tab:hyperparameters}
    \vspace{.5cm}
    \begin{tabular}{l|l}
        Parameter name & Value \\ \hline 
        Training seeds & 0, 1, 2 \\
        Number of online training timesteps & 4 million \\
        Size of the offline training replay buffer size & 1 million \\
        Number of offline training timesteps & 1 million \\
        Learning rate & 1e-4 \\ 
        Buffer size & 1 million \\
        Discount factor & 0.99 \\
        Target network update frequency & 1000 \\
        Batch size & 32 \\
        Epsilon for action selection & 0.01 \\
        Training frequency & 4 \\
        QDagger temperature & 1.0 \\
        Evaluation seeds & $0, 1, \dots, 9$ \\
        Offline training evaluation period & 100,000\\
        Online training evaluation period & 250,000 
    \end{tabular}
\end{table}

\section{More Examples of an Agent's Future Rewards}
Building upon Section 6.1, we provide more examples for the Atari \cite{bellemare2013arcade} Ms. Pacman, Breakout and Seaquest in Figures \ref{fig:mspacman-expected-reward}, \ref{fig:breakout-expected-reward-app} and \ref{fig:seaquest-expected-reward}. All agents in this section were trained with $N=40$ and $w=1$.

For Ms Pacman, we can observe from Figure \ref{fig:mspacman-expected-reward} that the agent expects high-frequency rewards of 1, around every three timesteps. Interestingly, the agent has high confidence in these future rewards (until $t+25$) as they are not spread out over several timesteps like Breakout or Seaquest in Figures \ref{fig:breakout-expected-reward-app} and \ref{fig:seaquest-expected-reward} respectively. At $t+25$ and beyond, the agent's confidence in a reward's particular timestep reduces, generally spread over two timesteps, e.g., $t+35$ and $t+36$.   

In comparison, Breakout and Seaquest (Figures \ref{fig:breakout-expected-reward} and \ref{fig:seaquest-expected-reward})show the agent's confidence in future rewards is significantly lower. For Breakout, the expected reward is spread over two timesteps with no subsequent expected reward within the next 30 timesteps. While, for Seaquest, the agent has very low confidence in the particular timestep of the reward, with the agent's belief centred around $t+5$. Both of these figures show that the agent's uncertainty about their future behaviour with future work required to understand if more training would enable the agent's greater confidence or if randomness within the environment prevents this.    

\begin{figure*}[t]
    \centering
    \includegraphics[width=\linewidth]{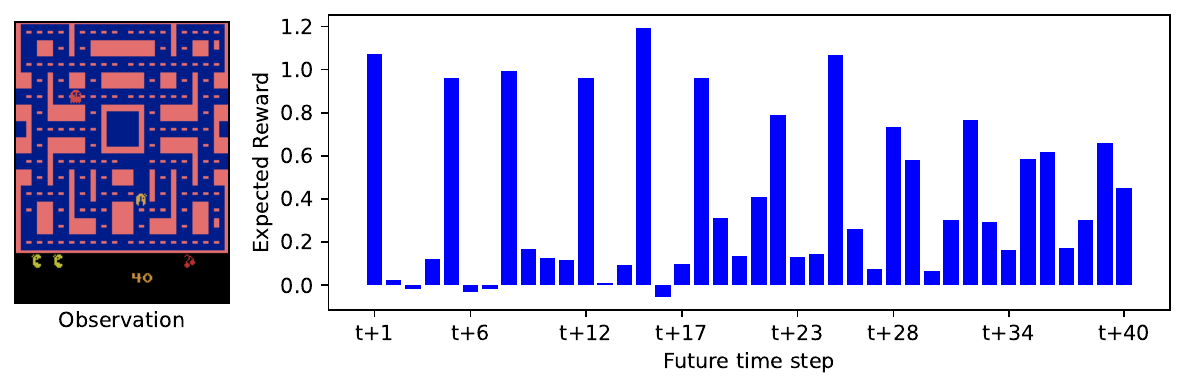}
    \caption{Expected Reward of a DQN-based TRD agent for Ms. Pacman with $N=40, w=1$.}
    \label{fig:mspacman-expected-reward}
    \vspace{.5cm}
\end{figure*}

\begin{figure*}
    \centering
    \includegraphics[width=\linewidth]{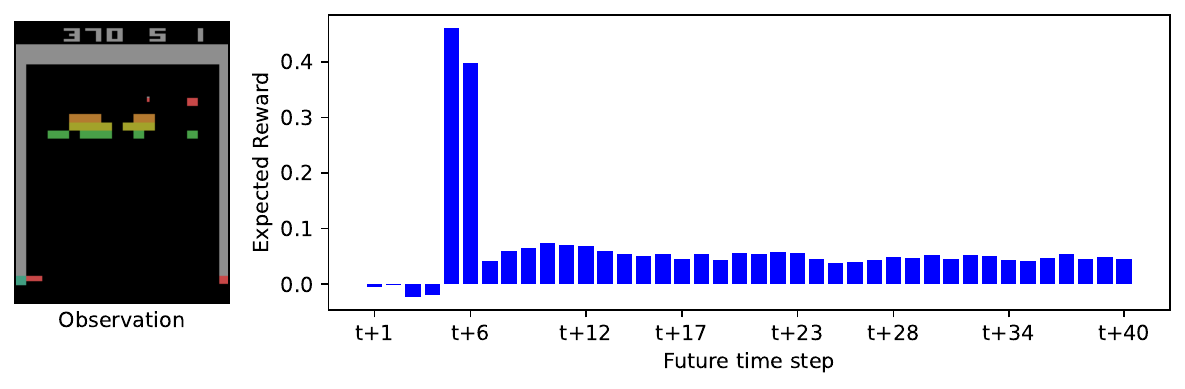}
    \caption{Expected Reward of a DQN-based TRD agent for Breakout with $N=40, w=1$.}
    \label{fig:breakout-expected-reward-app}
    \vspace{.5cm}
\end{figure*}

\begin{figure*}
    \centering
    \includegraphics[width=\linewidth]{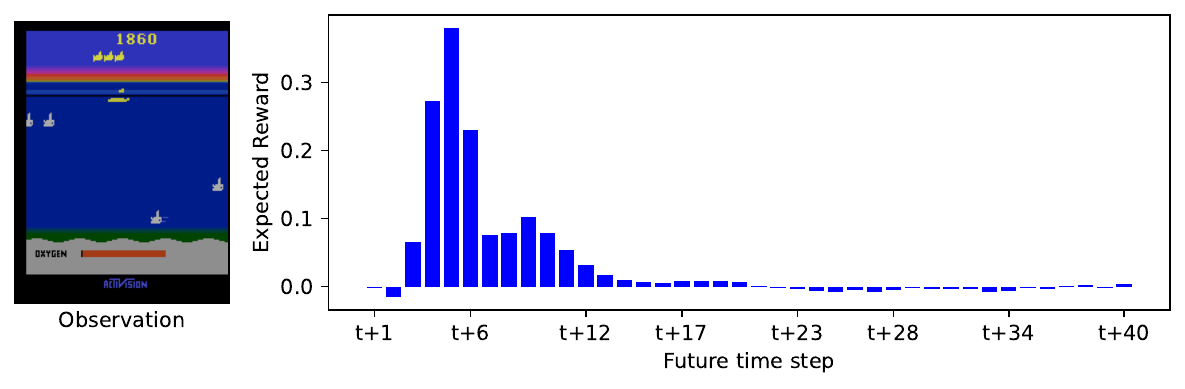}
    \caption{Expected Reward of a DQN-based TRD agent for Seaquest with $N=40, w=1$.}
    \label{fig:seaquest-expected-reward}
    \vspace{.5cm}
\end{figure*}

\section{More Visualisations of a Feature's Temporal Importance}
Building upon Section 6.2, in this Appendix, we provide more examples of feature importance for the Breakout environment at later stages in an episode than Figure 5 with Figures \ref{fig:breakout-feature-importance-middle} and \ref{fig:breakout-feature-importance-late}. These show that the agent's focus varies depending on the temporal distance to the predicted reward. 

With both figures, the agent shows greater relative importance to the bricks for $t+40$ than $t+1$, in particular, Figure \ref{fig:breakout-feature-importance-middle}. For Figure 5 and \ref{fig:breakout-feature-importance-middle}, we find that the ball has higher importance in $t+1$ than $t+40$; however, unexpectedly, in Figure \ref{fig:breakout-feature-importance-late}, this is reversed. We find that the agent's $t+40$ feature importance contains high importance of the ball while $t+1$ does not. The reason for this is unclear with future work required. 

\begin{figure*}
    \centering
    \includegraphics[width=\linewidth]{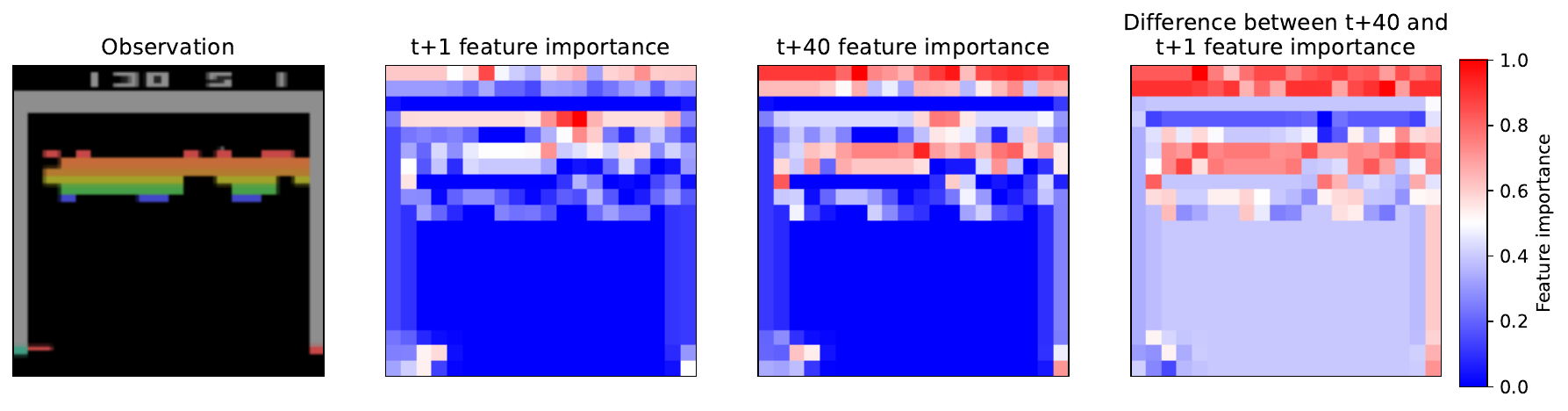}
    \caption{GradCAM saliency maps for the $t+1$ and $t+40$ expected reward along with their difference for a Breakout observation in the middle of an episode.}
    \label{fig:breakout-feature-importance-middle}
    \vspace{.5cm}
\end{figure*}

\begin{figure*}
    \centering
    \includegraphics[width=\linewidth]{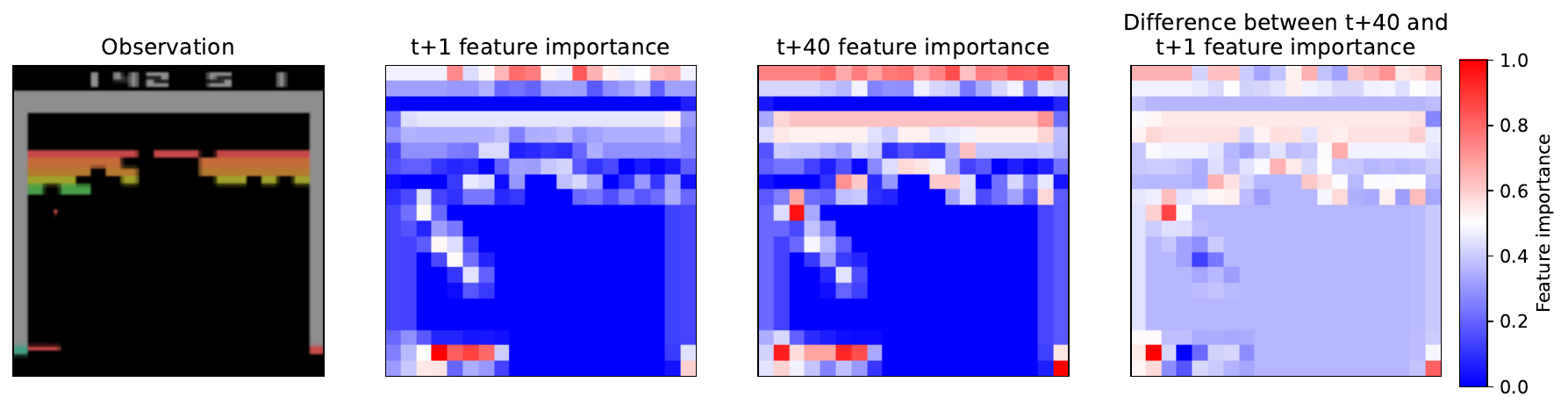}
    \caption{GradCAM saliency maps for the $t+1$ and $t+40$ expected reward along with their difference for a Breakout observation at the end of an episode.}
    \label{fig:breakout-feature-importance-late}
    \vspace{.5cm}
\end{figure*}

\section{More Contrastive Explanations of an Agent's Future Rewards}
Building upon Section 6.3, in this Appendix, we provide two more example contrastive explanations for the River raid and Ms Pacman environments in Figures \ref{fig:riverraid-contrastive} and \ref{fig:mspacman-contrastive}. These figures show the impact of an action on the agent's future expected rewards; in particular, they show a significant impact on the short-term rewards of the agents, which slowly converge to the same expected rewards. 

For Figure \ref{fig:riverraid-contrastive}, if the agent moves left, the difference in its future expected rewards compared to if it moved right have a four times increase in its expected reward, from $0.2$ to $0.8$. Similarly, for Figure \ref{fig:mspacman-contrastive}, the agent can move towards pellets (that reward the agent) with left or away, moving right. This is reflected in the agent's future beliefs in its rewards as the agent has significantly higher confidence only when moving left, whereas when moving right, the agent's expected rewards are limited to around $0.5$. 

\begin{figure*}
    \centering
    \includegraphics[width=\linewidth]{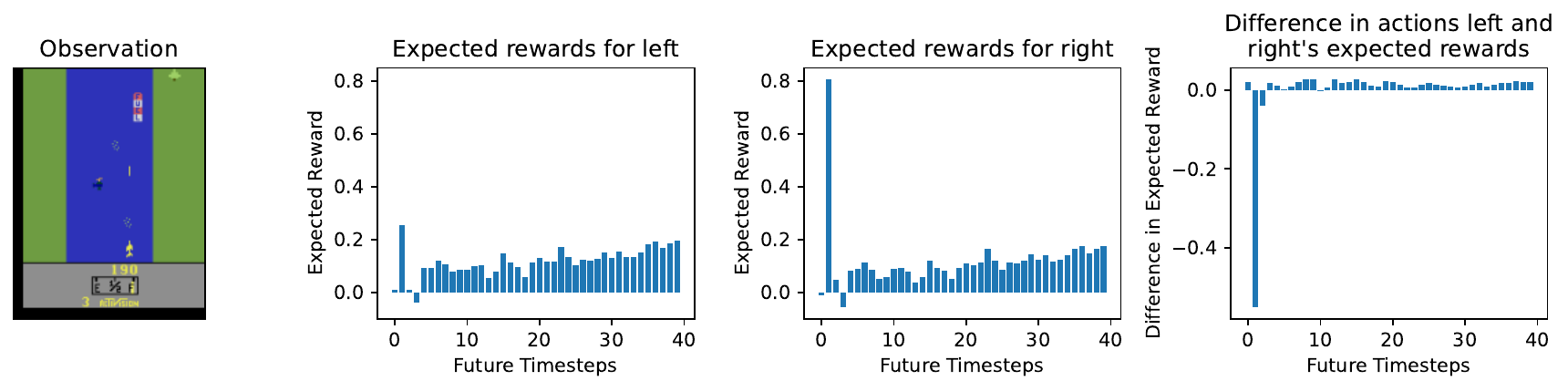}
    \caption{The difference of each future expected reward for taking Left and Right actions of the observation for the Atari River raid environment.}
    \label{fig:riverraid-contrastive}
    \vspace{.5cm}
\end{figure*}

\begin{figure*}
    \centering
    \includegraphics[width=\linewidth]{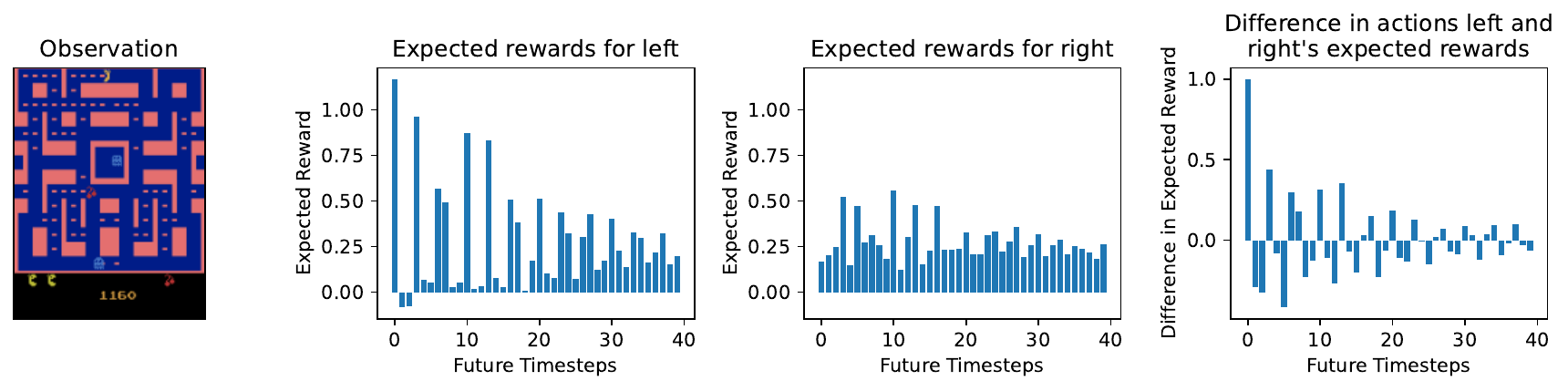}
    \caption{The difference of each future expected reward for taking Left and Right actions of the observation for the Atari Ms Pacman environment.}
    \label{fig:mspacman-contrastive}
    \vspace{.5cm}
\end{figure*}

\end{document}